\documentclass{article}
\usepackage{ijcai21}

\usepackage{times}
\usepackage{soul}
\usepackage{url}
\usepackage[hidelinks]{hyperref}
\usepackage[utf8]{inputenc}
\usepackage[small]{caption}
\usepackage{graphicx}
\usepackage{amsmath,amsthm,amssymb}
\usepackage{mathtools}
\usepackage{booktabs}
\usepackage{algorithm}
\usepackage{algorithmic}
\usepackage{xspace}
\usepackage{stmaryrd}
\usepackage{caption}
\usepackage{subcaption}
\usepackage{tikz}
\usetikzlibrary{trees}
\newtheorem{theorem}{Theorem}

\newtheorem{definition}{Definition}

\makeatletter
\newtheorem{repeatprop@}{Proposition}

\makeatother

\usepackage{xcolor}

\newcommand{\Lof}{\ensuremath{\mathcal{O}}}

\newcommand{\id}[1]{\llbracket{#1}\rrbracket}

\newcommand{\R}{\ensuremath{\mathbb{R}}}
\newcommand{\C}{\ensuremath{\mathcal{C}}}
\newcommand{\x}{\ensuremath{\mathbf{x}}}
\newcommand{\X}{\ensuremath{\mathbf{X}}}
\newcommand{\M}{\ensuremath{\mathbf{M}}}
\newcommand{\m}{\ensuremath{\mathbf{m}}}

\newcommand{\Y}{\ensuremath{\mathbf{Y}}}

\newcommand{\z}{\ensuremath{\mathbf{z}}}
\newcommand{\Z}{\ensuremath{\mathbf{Z}}}
\newcommand{\valX}{\ensuremath{\mathcal{X}}}

\newcommand{\EXP}{\ensuremath{\mathbb{E}}}

\newcommand{\SDP}[2]{\ensuremath{\text{SDP}_{\C,#1}(#2)}\xspace}
\newcommand{\EP}{\ensuremath{\text{EP}}\xspace}

\newcommand{\F}{\ensuremath{\text{EP}_{f}}\xspace}
\newcommand{\EL}{\ensuremath{\text{EP}_{\mathcal{O}}}\xspace}

\newcommand{\suffmetric}{{\ensuremath\mathcal{F}}\xspace}

\DeclareMathOperator*{\argmax}{argmax}

\newcommand{\citet}[1]{\citeauthor{#1}~[\citeyear{#1}]}

\title{Probabilistic Sufficient Explanations
}

\author{
Eric Wang
\and
Pasha Khosravi\And
Guy Van den Broeck\\
\affiliations
Department of Computer Science\\
University of California, Los Angeles\\
\emails
ericzxwang@ucla.edu, \{pashak,guyvdb\}@cs.ucla.edu
}

\begin{document}

\maketitle

\begin{abstract}
Understanding the behavior of learned classifiers is an important task, and various black-box explanations, logical reasoning approaches, and model-specific methods have been proposed.
In this paper, we introduce \emph{probabilistic sufficient explanations}, which formulate explaining an instance of classification as choosing the ``simplest'' subset of features such that only observing those features is ``sufficient'' to explain the classification. That is, sufficient to give us strong probabilistic guarantees that the model will behave similarly when all features are observed under the data distribution.
In addition, we leverage tractable probabilistic reasoning tools such as probabilistic circuits and expected predictions to design a scalable algorithm for finding the desired explanations while keeping the guarantees intact.
Our experiments demonstrate the effectiveness of our algorithm in finding sufficient explanations, and showcase its advantages compared to Anchors and logical explanations.

\end{abstract}

\section{Introduction}

Machine learning models are becoming ubiquitous, and are being used in critical and sensitive areas such as medicine, loan applications, and risk assessment in courts. Hence, unexpected and faulty behaviors in machine learning models can have significant negative impact on people. As a result, there is much focus on explaining and understanding the behavior of such models. Explainable AI (or XAI) is an active area of research that aims to tackle these issues.

There have been many approaches toward explaining an instance of a classification (called a local explanation) from different perspectives, including logic-based~\cite{symbolicExplain:IJCAI18,Ignatiev_abduction_2019,DarwicheHirth20a} or model-agnostic approaches~\cite{lime:kdd16,shap_NIPS2017,anchors:aaai18}. Each of these methods have their pros and cons; some focus on scalability and flexibility, and some focus on providing guarantees.

In this work, we strive to \emph{probabilistically} explain an instance of classification. Explanations are partial examples, where we treat the features not in the explanation as missing values. We aim for our explanations to be as simple as possible while providing the following sufficiency guarantee: given only the features in the explanation, with high probability under the data distribution $\Pr(\X)$, the classifier makes the same prediction as on the full example. Simplicity and sufficiency are often at odds with each other, hence balancing them is a challenging act. 

We briefly overview other approaches to local explanations and discuss their pros and cons in the framework of sufficiency and simplicity. Broadly, model-agnostic methods are more scalable and flexible but tend to be not as reliable as logic-based methods in providing sufficiency. On the other hand, logical explanation methods tend to sacrifice simplicity for deterministic guarantees of sufficiency. We then motivate how probabilistic notions of sufficiency can overcome these shortcomings as the foundation for sufficient explanations.

Next, we introduce the probabilistic reasoning tools needed to quantify sufficiency and define sufficient explanations: the \textit{Same-Decision Probability} (SDP) and \textit{Expected Prediction} (EP). We use probabilistic circuits (PCs) \cite{choi2020pc} to model the probability distribution $\Pr(\X)$ over the features due to their expressivity and tractability in answering complex probabilistic queries. We also explore connections between SDP and expected prediction and show why expected prediction is better suited for providing the probabilistic guarantees for sufficient explanations.

We then formalize the desired properties of sufficient explanations and motivate our choices for sufficiency and simplicity. Sufficiency leads to maximizing the expected prediction, while simplicity leads to constraining the size of explanations as well as choosing the subset that also maximizes the marginal probability. We capture all these with an optimization problem for finding sufficient explanations. Then, we design a scalable algorithm for finding the most likely sufficient explanation by leveraging tractability of expected prediction. 
    
We provide experiments showcasing the empirical advantage of sufficient explanations and the effectiveness of our search algorithm. Our advantages include: (i) compared to Anchors we get better and more accurate sufficiency guarantees and (ii) our method finds simpler explanations and is more scalable than logical methods.
Lastly, we show the tradeoffs between sufficiency and simplicity, and show that slightly reducing the guarantee can lead to simpler explanations with higher likelihood.

\section{Background and Related Work}
\label{sec:background}
\paragraph{Notation}
We use uppercase letters ($X$) for features (random variables) and lowercase letters ($x$) for their value assignments. 
Analogously, we denote sets of features in bold uppercase ($\X$) and their assignments in bold lowercase ($\x$).
We denote the set of all possible assignments to $\X$ as $\valX$. 
Concatenation $\X\Y$ denotes the union of disjoint sets. We focus on discrete features unless otherwise noted.

We represent a probabilistic predictor as $f:\valX\to [0,1]$ and its decision function as $\C:\valX\to\{0,1\}$,
with $T$ denoting the decision threshold. 
Hence, $\C(\x) = \id{f(\x)\geq T}$.
Sometimes we want to directly deal with log-odds instead of probabilities, in which case we use the log-odds predictor $\Lof:\valX\to\R$ which is defined as $\Lof(\x)=\log\frac{f(\x)}{1-f(\x)}$.

\subsection*{Related Work}
Computing explanations of classifiers has been studied from many different perspectives, including logical reasoning, black-box methods, and model-specific approaches. Some try to explain the learned model globally, making it more interpretable~\cite{survey_blackbox_explain,LiangAAAI19},
while others focus more locally on explaining its prediction for a single instance. 
Next, we go over some local explanation methods, discuss their pros and cons, and motivate how our framework might solve those issues.

\subsubsection*{Model Agnostic Approaches} These methods treat the classifier as a black box. Given an input instance to explain, they perturb the instance in many different ways and evaluate the model on those perturbed instances. Then, they use the results from the perturbations to generate an explanation. Two popular methods under this umbrella are Lime \cite{lime:kdd16} and SHAP \cite{shap_NIPS2017}. The main difference between these methods is the heuristics used to obtain perturbed instances and how to analyze the predictions on these local perturbations. Most provide feature attributions, which are real-valued numbers assigned to each feature, to indicate their importance to the decision and in what direction.

A benefit of these methods is that they can be used to explain any model and are generally more flexible and scalable than their alternatives. On the other hand, the downsides are that they can be very sensitive to the choice of local perturbations and might produce over-confident results \cite{ignatiev2019validating} or be fooled by adversarial methods \cite{advlime:aies20,dimanov2020you}. One of the main reasons for these downsides is that the distribution of the local perturbations tends to be different from the data distribution the classifier was originally trained on. Hence, these approaches do not benefit from the intended generalization guarantees of machine learning models. Moreover, some of the perturbations might be low probability or even impossible inputs, and we might not care as much about their classification outcome.

Additionally, feature attribution methods treat each feature independently and cannot easily capture interactions between bigger subsets of features such as when two features cancel each other's effects. We refer to \citet{camburu2020struggles} for more such examples and discussion on pros/cons of attribution based explanations.

\subsubsection*{Logical Reasoning Approaches} These methods provide explanations with some principled guarantees by leveraging logical reasoning tools. Some approaches use knowledge compilation and tractable Boolean circuits \cite{symbolicExplain:IJCAI18,DarwicheHirth20a,shi2020tractable}, some adopt the framework of abductive reasoning  \cite{Ignatiev_abduction_2019,ignatiev2019validating}, and some tackle a specific family of models such as linear models~\cite{marquessilva2020explaining}, decision trees \cite{izza2020explaining}, or tree ensembles \cite{devos2020additive}.

The main benefit of these approaches is that they guarantee provably correct explanations, that is they guarantee a certain prediction for all examples described by the explanation. On the other hand, one downside is that they are generally not as scalable (in the number of features) as black-box methods. Another downside is that they need to completely remove the uncertainty from the classifier to be able to use logical tools and therefore become more rigid. In particular, in order to guarantee a certain outcome with absolute certainty, it is often necessary to include many of the features into the explanation, making the explanation more complex.

Sufficient Reasons \cite{symbolicExplain:IJCAI18,DarwicheHirth20a} 
is one example of these methods that selects as an explanation a minimal subset of features guaranteeing that, no matter what is observed for the remaining features, the decision will stay the same. Sufficient reasons, as well as related logical explanations, ensure minimality and deterministic guarantees in the outcome, while as we see later our sufficient explanations ensure probabilistic guarantees instead.

For a recent and more comprehensive comparison of logic-based vs.\ model-agnostic explanation methods, we refer to \citet{ignatiev2019validating,trust_xai}.

\section{Motivation and Problem Statement}

We will overcome the limitations of both the model-agnostic and logic-based approaches by building local explanation methods that are aware of the distribution over features $\Pr(\X)$. We assume for now that we have access to an accurate model for the feature distribution, and discuss in Section~\ref{sec:prob-reason-tools} how it can be obtained. This distribution will allow us to (i)~reason about the classifier's behavior on realistic input instances and (ii)~provide probabilistic guarantees on the veracity of the explanations. We thus take a principled probabilistic approach in explaining an instance of classification. 

Intuitively, given an instance $\x$ and the classifiers outcome $\C(\x)$, we would like to choose a subset of features $\z \subseteq \x$ as the ``simplest sufficient explanation.'' Firstly, we want it to be sufficient, in that it provides strong probabilistic guarantees about the outcome of the classifier when only features $\z$ are observed. To this end, we want to maximize some sufficiency metric $\suffmetric(\z, \x, f, \Pr)$ describing the behavior of our predictor when only the explanation $\z$ is observed. However, naively maximizing this metric can run into the pitfall of building more complex explanations
in order to squeeze out tiny improvements in the sufficiency metric. To address this, we have our second metric of simplicity $S(\z, \Pr)$, which we use to define a constraint when maximizing the sufficiency metric. 
Putting these together, we call a subset $\z$ of an instance $\x$ a sufficient explanation if it maximizes some sufficiency metric $\suffmetric$ under some simplicity constraint.

With this current definition of sufficient explanations, there is still the possibility that two subsets of a given instance, one a subset of the other, are both simple and have the same degree of sufficiency. This motivates us to further define a minimal sufficient explanation $\z$, i.e. no subset of $\z$ is also a sufficient explanation. This condition is exactly the meaning of ``prime'' in prime implicants as they are used in logical explanations. We will see in Section~\ref{sec:prob-suff-exp} how we can use the feature distribution to ensure this minimality property.

\section{Probabilistic Notions of Sufficiency}
\label{sec:prob-reason-background}

We next introduce two notions of sufficiency using probabilistic reasoning tools and discuss their pros and cons. Then, we discuss their connections. Finally, we use these tools in Section~\ref{sec:prob-suff-exp} to formalize our definition and formulate finding the desired subsets of features as an optimization problem.

\subsection{Probabilistic Reasoning Tools}
\label{sec:prob-reason-tools}

Probabilistic reasoning is a hard task in general, so we need to choose our probabilistic model for $\Pr(\X)$ carefully. We choose \emph{probabilistic circuits} (PCs), which given some structural constraints enable tractable and exact computation of probabilistic reasoning queries such as marginals \cite{choi2020pc}. Moreover, they do so without giving up much expressivity. Another advantage of PCs is that we can learn their structure and parameters from data, allowing us to avoid the exponential worst-case behavior of other probabilistic models.

The two main probabilistic reasoning tools that we use for our explanations are the \textit{Same Decision Probability} (SDP) \cite{Chen2012TheSP} and \textit{Expected Prediction} (EP) \cite{KhosraviIJCAI19}. We introduce them next and explore their trade-offs and connections.

First we have SDP~\cite{Chen2012TheSP} which, intuitively, gives us the probability that our classifier has the same output as $\C(\x)$ given only some subset of observed features $\z$.\footnote{SDP was originally defined for the classifier being a conditional probability test in distribution $\Pr$. Here, we slightly generalize SDP to apply to a distribution $\Pr$ with a separate classifier~$\C$.}

\begin{definition}[Same Decision Probability]
    Given a classifier $\C$, a distribution $\Pr(\X)$ over features, a partition $\Z\M$ of features $\X$, and an assignment $\x$ to $\X$ (and corresponding $\z\subseteq\x$), the same decision probability (SDP) of $\z$ w.r.t.\ $\x$ is
    \begin{equation*}
        \SDP{\x}{\z}=\underset{{\m\sim \Pr(\M\mid\z)}}{\EXP} \id{\C(\z\m)=\C(\x)}.
    \end{equation*}
\end{definition}

The higher the SDP, the better guarantee we have that the partial example~$\z$ will be classified the same way as the full example~$\x$. SDP and related notions have been successfully used in applications such as trimming Bayesian network classifiers \cite{ChoiIJCAI17} and robust feature selection \cite{ChoiIJCAI18}. \citet{renooij2018same} introduced various theoretical properties and bounds on the SDP.

Other explanation methods that provide sufficiency guarantees can be fit into our framework, using SDP as the sufficiency metric. Notably, Anchors \cite{anchors:aaai18} can be thought of as an empirical approximation of SDP sufficiency, as they aim to provide sufficiency guarantees based on sampling local perturbations instead of relying on the data distribution. 
Logical explanations also fit in this framework, as they completely prioritize having deterministic guarantees, i.e. SDP=1, over the need for simplicity.
Another example of a method that uses probabilistic sufficiency can be found in \citet{KhosraviIJCAI19}, where they explain logistic regression models w.r.t.\ Naive Bayes data distributions.

Expected Prediction is another probabilistic reasoning task that has been shown to be successful in handling missing values in classification \cite{khosravi2019tractable,KhosraviIJCAI19,khosravi2020handling}. It provides a promising alternative for SDP in explanations. Intuitively, given some partial observation, expected prediction can be thought of as trying all possible ways of imputing the remaining features and computing an average of all subsequent predictions weighted by the probability of each imputation. In many cases, our classifiers directly output a probability and in those cases we can compute expected prediction as follows:

\begin{definition}[Expected Prediction]
\label{def:exp-pred}
Given a probabilistic predictor $f$, a distribution $\Pr(\X)$ over features, a partition $\Z\M$ of features $\X$, and an assignment $\z$ to $\Z$, the expected prediction of $f$ on $\z$ is
\begin{equation*}
    \F(\z) = \underset{\m\sim \Pr(\M \mid \z)}{\EXP} f(\z\m). 
\end{equation*}
\end{definition}

For some tasks, we may care more about odds rather than probabilities. In those cases, the predictor usually outputs the log-odds $\Lof(\x)$, so we also define expected log-odds:

\begin{definition}[Expected Log-Odds]
Given a log-odds predictor $\Lof$, a distribution $\Pr(\X)$ over features, a partition $\Z\M$ of features $\X$, and an assignment $\z$ to $\Z$, the expected log-odds of $f$ on $\z$ is
\begin{equation*}
   \EL(\z) = \underset{\m\sim \Pr(\M \mid \z)}{\EXP} \Lof(\z\m).
\end{equation*}
\end{definition}

In this paper, unless otherwise noted, expected prediction denoted as $\EP(\z)$ could refer to both cases.
 
\subsection{Connections between SDP and EP}
\label{sec:sdp-ep-relation}

The choice between using SDP and expected prediction as a sufficiency metric is important, as it will decide what kind of guarantees our explanations provide and how efficient they are to compute. Here we explore connections between SDP and expected prediction and provide intuition on some advantages of using expected prediction in defining sufficient explanations. 

One of the main differences between SDP and expected prediction is their computational complexity. While SDP is an appealing criteria to use for selecting explanations, computing the SDP exactly is computationally hard. In particular, it is $\mathit{PP}^\mathit{PP}$-hard on Bayesian networks \cite{choi2012same}. Even on a simple Naive Bayes model for both $\Pr$ and the classifier, computing SDP is NP-hard \cite{chen2013sdp}. 

On the other hand, expected predictions can be tractably computed for many different pairs of discriminative and generative models. For example, it is known to be tractable for the following cases: (i) logistic regression using a conformant naive Bayes distribution \cite{KhosraviIJCAI19}, (ii) decision trees w.r.t.\ probabilistic circuits (PCs) \cite{khosravi2020handling}, (iii) discriminative circuits w.r.t.\ PCs \cite{khosravi2019tractable}, and (iv) when both the feature distribution and predictor are defined by the same PC distribution $\Pr$. In the latter case, the predictor is the conditional probability $\Pr(c \mid \X)$, and the feature distribution is $\Pr(\X)$. Then, expected prediction can be reduced to probabilistic marginal inference in PCs which is tractable for decomposable circuits \cite{choi2020pc}.

Another difference comes from how SDP and expected prediction handle the two distinct uncertainties that arise from the feature distribution and the classifier. Although both are aware of the uncertainty from the feature distribution $\Pr$, SDP ignores the uncertainty in the classifier since it only deals with the decision function $\C(\x)$. For example, SDP cannot distinguish between cases when the classifier has low confidence $f(\x) = T + \epsilon$ (only slightly above decision threshold) and cases when the classifier has high confidence such as $f(\x)=0.99$. In many domains this distinction is vital. For example, doctors care more about the odds of a patient having cancer rather than binary decisions.

For these reasons, we choose to use expected predictions to define and optimize for our explanations. The good news is that optimizing $\EP$ also allows us to maximize a lower bound on SDP. The next theorem provides a theoretical lower bound for SDP using expected predictions. 

\begin{theorem}
\label{thm:sdp-ep-ineq}
Given a predictor $f$ (or $\Lof$), its thresholded classifier $\C$, a positively classified instance $\x$ (i.e. $\C(\x)=1$), a distribution $\Pr$, and some subset of the features $\z \subseteq \x$, we have:
\begin{equation}
    \label{eq:ep-sdp}
    \SDP{\x}{\z}>\frac{\EP(\z)-T}{U(\z)-T}.
\end{equation}
where $\EP(\z)$ refers to expected prediction, and $U(\z)$ is an upper bound for the predictor after fixing $\z$, i.e. $\forall \m\, U(\z)\geq f(\z\m)$. Moreover, if $U(\z)$ is a tight bound, then Equation~\ref{eq:ep-sdp} is also tight.
\end{theorem}

\begin{proof}
The proof is included in the appendix.\footnote{Available at \href{http://starai.cs.ucla.edu/papers/WangIJCAI21.pdf}{http://starai.cs.ucla.edu/papers/WangIJCAI21.pdf}}
\end{proof}

Theorem~\ref{thm:sdp-ep-ineq} gives us two ways of computing a lower bound for SDP: using a probabilistic predictor $f$ or a log-odds predictor $\Lof$. Additionally, the theorem can be easily generalized to cases when $\C(\x)=0$. See appendix for a more detailed discussion on which bound is better.

\section{Probabilistic Sufficient Explanations}
\label{sec:prob-suff-exp}

In this section, we use the aforementioned sufficiency metrics and introduce some desirable constraints to ensure simplicity. Then, we put everything together to formalize sufficient explanations and introduce an optimization problem for finding them.
Finally, we introduce a search algorithm for finding sufficient explanations by modifying beam search and leveraging tractability of expected predictions.

To simplify the definitions, we assume without loss of generality that the instance we want to explain is positively classified, i.e. $\C(\x) = 1$. In Section~\ref{sec:sdp-ep-relation}, we saw two candidates that we can use for probabilistic sufficiency guarantees. The first was SDP, which focuses on the final decision of the model. The second was expected prediction, which takes the confidence of the model into account. Because of this, along with its computational tractability, we choose to use expected prediction for our sufficiency guarantees. More specifically, we want to maximize the expected prediction of our explanation to ensure our model is confident in its classification. In addition, maximizing the expected prediction also maximizes a lower bound for SDP, so in practice we will still get good SDP guarantees.

Having chosen a suitable sufficiency metric, we must now choose a simplicity constraint. There are multiple candidates to choose from for this as well. For example, we can impose a cardinality constraint or require explanations to have a high enough likelihood. We choose the cardinality constraint as it is easy to decide on a threshold on the explanation size. We can now formalize the notion of sufficient explanations.

\begin{definition}[Sufficient Explanations] Given a predictor $f$, distribution $\Pr$, and a positively classified instance $\x$ (i.e. $\C(\x)=1$), the set of sufficient explanations for $\x$ is defined as the solution of the following optimization problem:
\begin{gather*}
    \argmax_{\z \, \subseteq \, \x} \EP(\z)  \text{~~~~~s.t. } |\z| \leq k
\end{gather*}
which we denote as $\text{SE}_{k}(\x)$.
\end{definition}

Having defined sufficient explanations, and following the same line of thought as before, we are now interested in choosing sufficient explanations which are minimal. For logical explanations, any minimal sufficient subset can be chosen. However, we can be more selective in our choices using the feature distribution. A natural choice is to choose the most likely sufficient explanations, as these are the most realistic. By maximizing the marginal probability of the explanation, we also ensure the desired minimality. This is because, for subsets $\z_1$ and $\z_2$ of an instance $\x$, if $\z_1\subseteq\z_2$ then $\Pr(\z_1)\geq\Pr(\z_2)$. We thus arrive at the following definition.

\begin{definition}[Most Likely Sufficient Explanations]
Given a predictor $f$, a distribution $\Pr$, and an instance $\x$, the most likely sufficient explanations for $\x$ are given by:
\begin{gather*}
    \text{MLSE}_{k}(\x) = \argmax_{\z \, \in \, \text{SE}_{k}(\x)}\ \Pr(\z)
\end{gather*}
\end{definition}

Next, we devise a search algorithm for finding sufficient explanations. In Section~\ref{sec:experiments}, we show that our algorithm works well in practice, ensuring both sufficiency and simplicity.

\subsection*{Finding Probabilistic Sufficient Explanations}
\label{sec:compute-suff-exp}

To find the most likely sufficient explanations, we use a beam search algorithm to greedily search the space of potential explanations. We do so by keeping track of the top $b$ candidates (the beam) for explanations for each cardinality based on their expected predictions and marginal probabilities.

In more detail, we begin at level zero with the empty candidate, i.e. no features selected. For each subsequent level, we expand the top $b$ candidates of the previous level by considering all feature subsets with one more feature than in the previous level, formed by adding a previously unselected feature to each candidate. We then select the new top $b$ candidates by ranking the expanded states based on $\EP$, breaking ties using $\Pr$, and keeping the top $b$. The search stops after level $k$. At the same time we keep track of the most likely candidate satisfying the current maximum \EP.

Due to the nature of the search, each level of the beam search is highly parallelizable, which can be leveraged to speed up the search. Additionally, by running the algorithm for $k$ levels, we can also keep track of the best sufficiency guarantee at each level between 1 and $k$. Hence, with a little extra book keeping we can keep track of explanations with different sizes, and thus degrees of simplicity, and their corresponding sufficiency guarantees.

Note that the search framework can also be easily adapted for different use cases. For example, we can provide a fixed sufficiency constraint and then find the simplest and most likely explanation with an expected prediction higher than the sufficiency threshold. However, choosing a good threshold for expected prediction is not always straightforward.

\section{Experiments}
\label{sec:experiments}

In this section, we provide several experiments to showcase the effectiveness of our search algorithm in finding sufficient explanations.\footnote{Code at \href{https://github.com/UCLA-StarAI/SufficientExplanations}{github.com/UCLA-StarAI/SufficientExplanations}} 
Additionally, we aim to highlight the advantages of our method in comparison with Anchors and logical explanation methods. More specifically, we would like to answer the following questions:

\begin{itemize}
    \item[--] Can we find explanations with good sufficiency guarantees? How do they compare with Anchors?
    \item[--] Does relaxation from logical to probabilistic guarantees lead to much simpler explanations?
    \item[--] What are the tradeoffs between different sufficiency levels and explanation complexity?
\end{itemize}

We use the adult and MNIST datasets~\cite{kohavi1996scaling,yann2009mnist} for our experiments.
For each dataset, we model the feature distribution by learning a probabilistic circuit (PC) using the open source Juice library~\cite{DangAAAI21}.  We choose decision forests learned by XGBoost \cite{xgboost} as our classifier, as they are popular and because expected prediction is tractable for forests w.r.t.\ PCs \cite{khosravi2020handling}. For more detailed information on the datasets, preprocessing steps, learned models, and computing infrastructure, please refer to the appendix.

\begin{figure}[t]
    \centering
    \begin{subfigure}{\linewidth}
        \centering
        \includegraphics[width=0.7\linewidth]{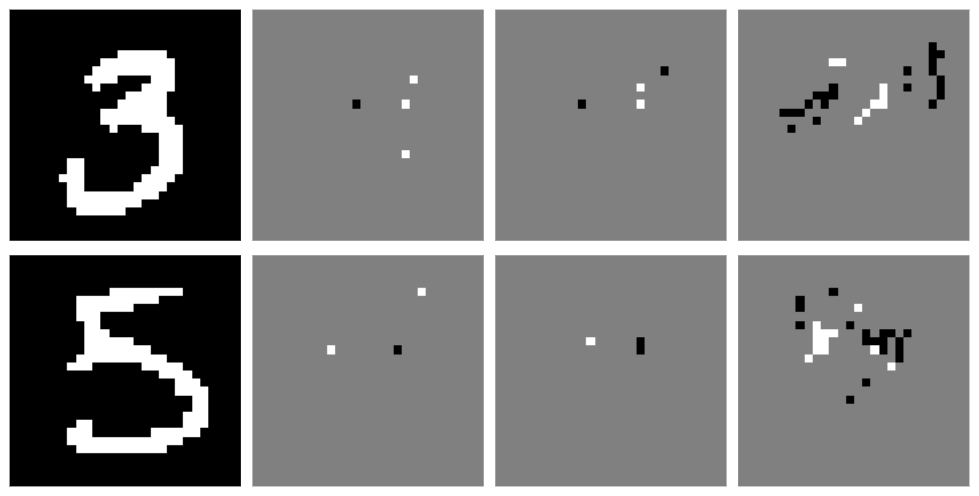}
        \subcaption{Correctly classified examples}
    \end{subfigure}
    \begin{subfigure}{\linewidth}
        \centering
        \includegraphics[width=0.7\linewidth]{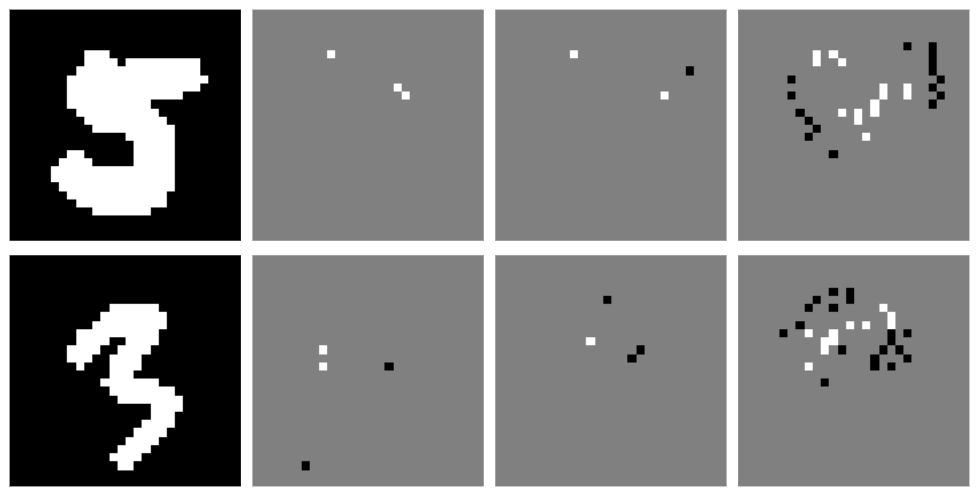}
        \subcaption{Misclassified examples}
    \end{subfigure}
    \caption{Explanations for selected MNIST images. From left to right: 1) original image; 2) Anchors explanation; 3) our explanation with same number of features 4) our explanation with $k=30$. Gray pixels were not chosen for the explanation. Pixels chosen for the explanation are colored the same color as the original image.}
    \label{fig:mnist}
\end{figure}
\begin{table}[t]
    \centering
    \begin{tabular}{c c c c }
        Method & $|\EL(\z)|$ & $\SDP{\x}{\z}$ & $\log P(\z)$ \\
        \hline
        Anchors & $0.75\pm0.37$ & $0.66\pm0.08$ & $-3.29\pm0.88$\\
        MLSE$_s$ & $1.57\pm0.29$ & $0.86\pm0.05$ & $-3.05\pm0.65$\\
        MLSE$_{10}$ & $3.11\pm0.23$ & $0.99\pm0.01$ & $-6.98\pm1.37$\\
        MLSE$_{20}$ & $3.60\pm0.15$ & $1.00\pm0.00$ & $-9.90\pm2.14$\\
        MLSE$_{30}$ & $3.75\pm0.13$ & $1.00\pm0.00$ & $-11.77\pm2.88$
    \end{tabular}
    \caption{Comparison of average expected log-odds, SDP, and marginals between Anchors and MLSE, averaged over 50 random MNIST test images. We take the absolute value of $\EL(\z)$ to measure confidence of the explanations (since it could be negative). MLSE$_s$ sets the cardinality constraint to the same size of the Anchors explanation for each image. The $\pm$ denotes one standard deviation. The SDP values are approximated.
    }
    \label{tab:mnist_sdp_and_ep}
\end{table}

\subsection{Comparison with Anchors}

To demonstrate the scalability of our method and showcase some advantages of our method in comparison with Anchors, we ran our algorithm on the binarized MNIST dataset with a binary classification task of distinguishing between 3s and 5s.

Some images are shown in Figure \ref{fig:mnist} along with a comparison between explanations found using our method and Anchors. For Anchors, we used an SDP (precision) threshold of 0.95, a tolerance ($\delta$) of 0.05, and a beam size of 5. On average it took Anchors 454s to generate explanations. Our algorithm with the same beam size and cardinality constraint $k=30$ took 347s using 16 threads; the sufficient explanations with the same size as Anchors took 45s. See appendix for more details on the run-times.

The last image of each row in Figure~\ref{fig:mnist} is the explanation found using our algorithm with cardinality constraint $k=30$. We see that these explanations were able to pick up on certain features we would naturally use to distinguish between 3s and 5s. In particular the chosen pixels were mostly in the upper portion of each image. This makes sense as both 3s and 5s have a similar arch shape in their lower portions, so the upper portion would be more useful for distinguishing between the two. Additionally, the explanations contain not only some white pixels showing an outline of the predicted number, but also some black pixels, where a number of the opposite label may be present. Finally, by looking only at the explanation in the rightmost column of the last two rows we can guess that the classifier will misclassify those examples.

Table \ref{tab:mnist_sdp_and_ep} provides data for the expected log-odds, SDP, and marginal probabilities for generated explanations using Anchors and our method in a few different scenarios. Since SDP is intractable to compute exactly, we estimate it by computing the SDP on $10000$ samples drawn from the probabilistic circuit conditioned on the explanation. One advantage of using PCs as our generative model is that drawing conditional samples from $\Pr(\cdot \mid \z)$ is very fast. For example, generating 10K samples takes 1 second.
We see that the Anchors are quite overconfident, giving explanations with much lower sample SDPs compared to the desired 0.95. This trend was also observed in \citet{ignatiev2019validating}.

\begin{figure}[t]
    \centering
    \includegraphics[width=0.7\linewidth]{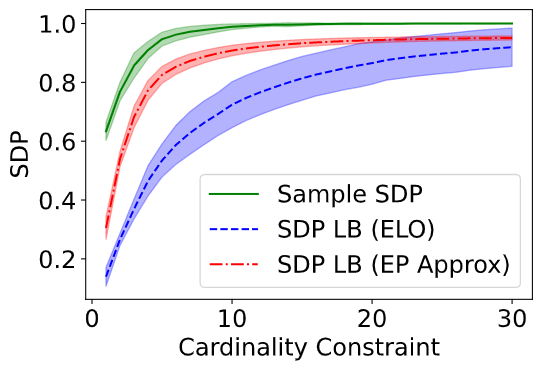}
    \caption{For sufficient explanations (sizes 1-30), we plot SDP estimates (green) vs SDP lower bounds calculated based on expected log-odds (blue), and lower bounds based on approximate expected prediction (red), averaged over 50 test images of MNIST.  Shaded regions represent one standard deviation.
    }
    \label{fig:sdp_chart}
\end{figure}

We plot the sample SDPs for our explanations, along with lower bounds calculated using Theorem~\ref{thm:sdp-ep-ineq}, in Figure~\ref{fig:sdp_chart}. The green line shows that, even when optimizing the EP of our explanations, the SDP still tends to be very high. Moreover, the blue line shows that the simple and efficient SDP lower bound can also provide this guarantee for some of the larger explanations with high EP. The red line is yet another way to estimate the SDP lower bound. See appendix for more details.

\subsection{Comparison with Logical Explanations}

We also compared explanations found using our method to logical explanations, i.e. minimal explanations with $\text{SDP}=1$. We used the adult dataset, which has much fewer features, for this task in order for the logical explanation computation to become feasible. We removed some features to allow for brute force computation of minimum cardinality logical explanations. We chose to use brute force because, as far as we know, there is no method for finding logical explanations for our use case. In total we removed 3 features, leaving us with 11 features. From our brute force search on some test examples, we found that logical explanations needed on average 39\% of the features in order reach an SDP of one. By using our algorithm of maximizing the expected prediction, we found that in most cases selecting only 18\% of the features was already enough to guarantee an SDP of 0.95 on average. More detailed numbers are provided in Table~\ref{tab:logical_chart}. We see that the strict requirement imposed by logical explanation methods can lead to selecting more features leading to more complex explanations. We expect this gap between the complexities of logical and probabilistic explanations to widen for datasets with more features or more complex models.

\begin{table}[t]
    \centering
    \begin{tabular}{c c c c }
        Method & $\SDP{\x}{\z}$ & $\log P(\z)$ & size\\
        \hline
        Logical & $1.0$ & $-7.12\pm2.11$ & $4.30\pm1.13$\\
        Anchors & $0.98\pm0.02$ & $-4.27\pm2.61$ & $2.02\pm1.26$\\
        MLSE$_{s}$ & $0.97\pm0.03$ & $-3.85\pm2.37$ & $2.17\pm1.18$\\
        MLSE$_{1}$ & $0.88\pm0.19$ & $-2.23\pm1.34$ & $1.0\pm0.0$\\
        MLSE$_{2}$ & $0.95\pm0.08$ & $-3.88\pm1.88$ & $2.0\pm0.0$\\
        MLSE$_{3}$ & $0.98\pm0.05$ & $-4.77\pm2.31$ & $2.99\pm0.06$\\
        MLSE$_{4}$ & $0.99\pm0.03$ & $-5.63\pm2.59$ & $3.96\pm0.22$\\
    \end{tabular}
    \caption{SDP, marginal probability, and size statistics for explanations found using different methods for the adult dataset.}
    \label{tab:logical_chart}
\end{table}

\subsection{Tradeoffs between Sufficiency and Simplicity}
\label{sec:tradeoff-suff-complex}

\begin{figure}[t]
    \centering
    \includegraphics[width=\linewidth]{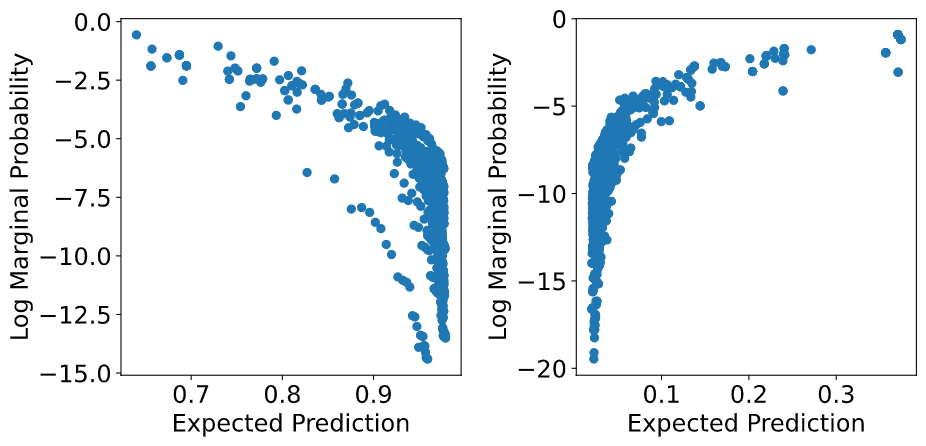}
    \caption{Tradeoff between expected prediction and marginal probability for MLSEs. The first plot is for positive label images (5s); the second is for negative label images (3s). Expected predictions were computed using a first order approximation.}
    \label{fig:tradeoff}
\end{figure}

Finally, we examine the tradeoff between maximizing sufficiency and explanation simplicity. We compare the expected predictions with marginal likelihoods of explanations. Figure \ref{fig:tradeoff} shows a scatter plot comparing these two quantities for different explanations generated for 50 MNIST images. As we see, the general trend is that by enforcing less strict sufficiency requirements we can get more likely (and also smaller) explanations. In particular, the trend is very steep around EPs of 0 and 1, meaning that making the guarantees even slightly probabilistic will lead to significant simplification of the explanations, thus validating our probabilistic approach.

\section{Conclusion}
\label{sec:conclusion}

We introduced a new framework for reasoning about the local behavior of classifiers. We formulated the problem as finding simplest and most likely explanations that maximize a probabilistic guarantee, and discussed advantages of our framework compared to model agnostic and logical explanation methods. We provided experiments to validate our claims. We conclude that probabilistic sufficient explanations are a valuable addition to the arsenal of local explanation~methods.

\section*{Acknowledgments}
The authors would like to thank YooJung Choi for helpful discussions regarding SDP.
This work is partially supported by NSF grants \#IIS-1943641, \#IIS-1633857, \#CCF-1837129, DARPA grant \#N66001-17-2-4032, a Sloan Fellowship, Intel, and Facebook.

\bibliographystyle{named}
\bibliography{ijcai21}

\clearpage
\appendix
\section{Appendix}
\subsection{Proof of Theorem~\ref{thm:sdp-ep-ineq}}

\begin{proof}
    We provide the proof for the case of using $\EP = \F$, and $U(\z)$ providing an upper bound on the probabilistic predictor $f(\z\M)$. The proof is mostly identical in the case of having $\EP = \EL $ and $U(\z)$ being an upper bound for the log-odds predictor $\Lof(\z\M)$ .
    
    Let $T$ be the decision threshold. Then we have $\SDP{\x}{\z}=\Pr(f(\z\m)\geq T)$ and $\F(\z)=\EXP[f(\z\m)]$ where $\m\sim\Pr(\M|\z)$. Thus,
    \begin{align*}
    \F(\z)&=\EXP[f(\z\m)]\\
    &=\EXP[f(\z\m)|f(\z\m)<T] \cdot \Pr(f(\z\m)<T)\\
    & \quad+ \EXP[f(\z\m)|f(\z\m)\geq T] \cdot \Pr(f(\z\m)\geq T)\\
    &< T(1-\Pr(f(\z\m)\geq T))\\
    & \quad+ U(\z)\Pr(f(\z\m)\geq T)\\
    &= T + (U(\z)-T) \Pr(f(\z\m)\geq T)\\
    &= T + (U(\z)-T)\ \SDP{\x}{\z}.
    \end{align*}
    Rearranging the terms leads to Equation~\ref{eq:ep-sdp}. 
    
    Next we show how to construct a distribution for $f(\z\M)$ to demonstrate the tightness of the bound. Assume we are given $\F(\z)=F$, $U(\z)=U$, and a decision threshold $T$ with $F<T<U$. Then let $\epsilon>0$ and consider the distribution given by 
    \begin{align*}
        P(f(\z\M)=k)=
        \begin{dcases}
        \frac{F-T}{U-T}+\epsilon&\text{if }k=U\\
        \frac{U-F}{U-T}-\epsilon&\text{if }k=A\\
        0&\text{otherwise}
        \end{dcases}
    \end{align*}
    where $A$ is some constant which makes $\F(\z)=F$ and $A<T$ so that $\SDP{\x}{\z}=\frac{F-T}{U-T}+\epsilon$. We now show how to solve for $A$. 
    
    We start by computing expected prediction (Definition~\ref{def:exp-pred}), we have:
    \begin{align*}
        \F(\z) &= \underset{\m\sim \Pr(\M \mid \z)}{\EXP} f(\z\m) \\
              &= U \times P(f(\z\M) = U) + A \times P(f(\z\M) = A) + 0 \\
              &= U\left(\frac{F-T}{U-T}+\epsilon\right)+A\left(\frac{U-F}{U-T}-\epsilon\right) \\
              &= F
    \end{align*}
    and then solve for $A$:
    \begin{align*}
        A &= \frac{F-U\left(\frac{F-T}{U-T}+\epsilon\right)}{\left(\frac{U-F}{U-T}-\epsilon\right)}\\
        &= \frac{F(U-T)-U(F-T)-U(U-T)\epsilon}{U-F-(U-T)\epsilon}\\
        &= \frac{T(U-F)-U(U-T)\epsilon}{U-F-(U-T)\epsilon}\\
        &< \frac{T(U-F)-T(U-T)\epsilon}{U-F-(U-T)\epsilon}\\
        &= T
    \end{align*}
    Thus we have $\F(\z)=F$ and, since $A<T$, $\SDP{\x}{\z}=\frac{F-T}{U-T}+\epsilon$.
\end{proof}

\subsection{Beam Search Algorithm}

\subsubsection{Pseudocode}

\begin{algorithm}[ht]
\caption{Finding Most likely Sufficient Explanations}
{\bfseries Input:} instance $\x$, feature distribution $\Pr(\z)$, expected prediction function $\EP(\z)$, max features $k$, beam size $b$

\begin{algorithmic}[1]
\label{alg:find-exp}
\STATE $\text{MLSE} \gets \emptyset$
\STATE $\text{beam} \gets \{\emptyset\}$
\FOR {$i=1...k$}
    \STATE $\text{candidates}\gets\bigcup_{c\in\text{beam}}\text{expand}(c)$
    \STATE $\text{compute }\EP(c)\text{ and }\Pr(c)\text{ for }c\in\text{candidates}$
    \STATE $\text{beam}\gets\text{top\_}b(\text{candidates})$
    \STATE $\text{update MLSE}$
\ENDFOR
\end{algorithmic}
\end{algorithm}

In the above pseudocode, the expand function generates a new set of candidates, each formed by adding one yet unobserved feature to the input candidate. The top\_$b$ function selects the $b$ candidates with the highest $\EP$, with ties broken using $\Pr$. Finally, the MLSE is updated if an explanation with a higher $\EP$ is found, or if one has the same $\EP$ as the previous MLSE but a higher $\Pr$.

\subsubsection{Computational Complexity}

When using PCs to model the feature distribution,
the runtimes for tractable expected prediction (\EP) and marginal ($\Pr$) algorithms do not depend on how many features are observed. $\Pr$ can be computed in linear time w.r.t. circuit size. Linearity of decision forest classification also allows for expected prediction computations in time linear w.r.t. circuit size. Thus we treat \EP and $\Pr$ as oracles and measure our algorithm's runtime by the number of calls to them. In each level of the beam search, we expand the previous level's top $b$ candidates, giving us at most $nb$ unique states, where $n$ is the number of features. We then call $\EP$ and $\Pr$ once for each state. Since the search stops after level $k$, overall we make in total $O(nkb)$ calls to the $\EP$ and $\Pr$ oracles.

\subsection{More Experiment Details}

\subsubsection{Computing Infrastructure} 

All experiments were run on a Linux server with 40 CPU cores and 500 \text{GB} of RAM, albeit not all the memory or CPU cores where needed for generating one explanation. Our method is highly parallizable as we see adding more threads speeds up the process substantially. No GPUs were used in this paper.

Although we did not use any GPUs, it should be easy to re-implement some parts of our method to support GPUs. Some of the main ingredients of our method such as computing marginal probabilities using Probabilistic Circuits already support GPUs and enjoy huge speed ups. Furthermore, computing expected predictions and the search at each level should be amenable to GPU parallization. So, overall we should be able to speed up the algorithm by 1-2 orders of magnitude. This allows us to scale to even more complicated models with more features in the future.

\subsubsection{Datasets and Preprocessing Steps} 

The MNIST \cite{yann2009mnist} dataset consists of 60,000 grayscale images ($28\times28$ pixels) of handwritten digits (0-9). We limit the dataset into digits of 3 and 5 to ensure we have a binary classification task. The dataset is already split between train and test images, so we used the same split, selecting from each digits of 3s and 5s. Since our probabilistic circuit library currently only supports binary feature values, we also binarized pixel values. indepedently for each image, by applying a threshold at 0.05 standard deviations above the mean. 

The classification task for the adult dataset \cite{kohavi1996scaling} is to determine whether a given individual makes over $\$50,000$ per year. Features include age, sex, working class, hours worked per week, education level,
nationality, etc. We perform discretization of continuous features by applying standard binning. We then perform one-hot-encoding for all features. Again, this is because the probabilistic circuit library only supports binary features. We then used an 80-20 split for our training and test data. In general, our method should be straightforward to extend to multi-category features and also continuous features. 

Source code for processing the datasets will be included in the final version of paper as supplementary.

\subsubsection{Details on models}

\begin{figure}
    \centering
    \includegraphics[width=\linewidth]{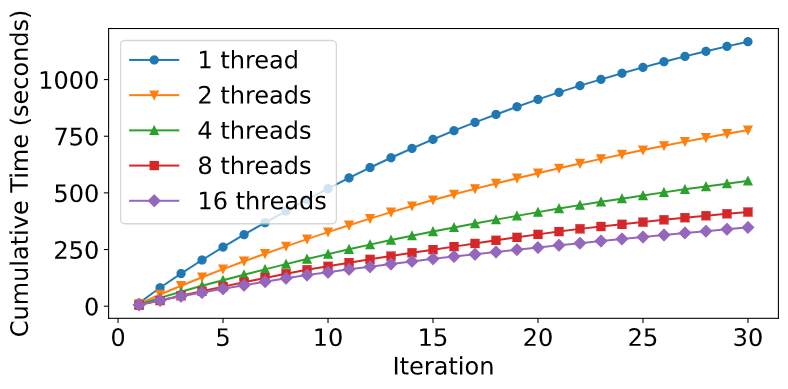}
    \caption{MNIST Runtimes. Average cumulative time taken until completion of each iteration of the beam search algorithm for different numbers of threads.}
    \label{fig:mnist_threads}
\end{figure}

For our MNIST experiments, we used a PC with 10124 nodes, 15640 edges, and 5916 parameters. It took approximately 35 minutes to train, stopping after 410 iterations. For our classifier, we trained a decision forest with 6 trees, each with a maximum depth of 6. The final model has a total of 297 leaves. The classification accuracy on our test images is $98.63\%$.

For our adult experiments, we used a PC with 22899 nodes, 44152 edges, and 14828 parameters. It took approximately 17 minutes to train, stopping after 190 iterations of structure learning. For our classifier, we trained a decision forest with the same constraints as above. The final model has a total of 255 leaves. The classification accuracy on our test dataset is $84.20\%$.

\subsubsection{Metric Calculation Details}

The classifiers used in our experiments were decision forests, which allow for efficient computation of expected log-odds. Classification using these models is done by summing the weights of one leaf from each tree, where the chosen leaf from each tree has a path from the root not contradicting with the instance to be classified. This additive model allows us to utilize linearity of expectation to compute expected log-odds \cite{khosravi2020handling}.

This model does not, however, allow for efficient computation of expected predictions, where prediction refers to taking the sigmoid of the classifier output. This is due to the non-linear nature of the sigmoid function. Instead, when we give expected prediction values, we use a first order approximation by taking the sigmoid of the expected log-odds.

In Figure~\ref{fig:sdp_chart} we show a plot with SDP lower bounds computed using Theorem~\ref{thm:sdp-ep-ineq}. Graphing of the blue curve requires computation of an upper bound $U(\z)$ on the predictions of a decision forest. The upper bound used is a loose one which we computed by adding the weights of one leaf from each tree in the forest, where each leaf is the maximum weighted leaf in its tree whose path from the root does not contradict with the explanation. This bound is a loose once since paths for leaves from different trees may contradict. For the red curve, since our model outputs log-odds, we used the first order approximation mentioned above.

\subsubsection{Effects of Parallelism on Performance}

One advantage of our method is that we can parallelize computation of expected predictions at each level of the beam search. The effects of can be seen in Figure~\ref{fig:mnist_threads}, which plots the cumulative time taken to complete each level of the beam search for different numbers of threads. We see that more threads do indeed allow for much faster runtimes. We leave investigating benefits of GPU acceleration for future work.

\subsubsection{More on SDP Lowerbounds}
In the main text we saw that Theorem~\ref{thm:sdp-ep-ineq} gives us two ways of computing a lower bound for SDP: using a probabilistic predictor $f$ or a log-odds predictor $\Lof$. One of them might give a tigher bound, depending on how our predictor is defined. Generally, the tighter bound we can find for $U(\z)$ the tighter bound we get for SDP.

For example, in the case of having a probabilistic predictor, the decision threshold is usually $T=0.5$, and we can easily get a trivial upper bound of $U(\z)=1$. Depending on the model family we might be able to get a tighter bound for $U(\z)$, and hence a tighter lower bound for SDP.
For the case of a log-odds predictor, we usually take $T = 0$, but getting an upper-bound $U(\z)$ might not be trivial. However, in many cases an upper bound $U(\z)$ can be computed efficiently for log-odds predictor, but might not be tight.

\subsubsection{Limitations}

Some knowledge of probabilities is needed to fully interpret the results. However, as seen with the MNIST example, visual representations can be designed to give users a feel for what the model is doing.

\subsubsection{More MNIST examples}

Figure~\ref{fig:more_mnist} provides some extra examples on sufficient explanations for MNIST. We see examples for both correctly and incorrectly classified images, and see the effects of cardinality constraints on the explanations.

\begin{figure}
    \centering
    \begin{subfigure}{\linewidth}
        \centering
        \includegraphics[width=0.7\linewidth]{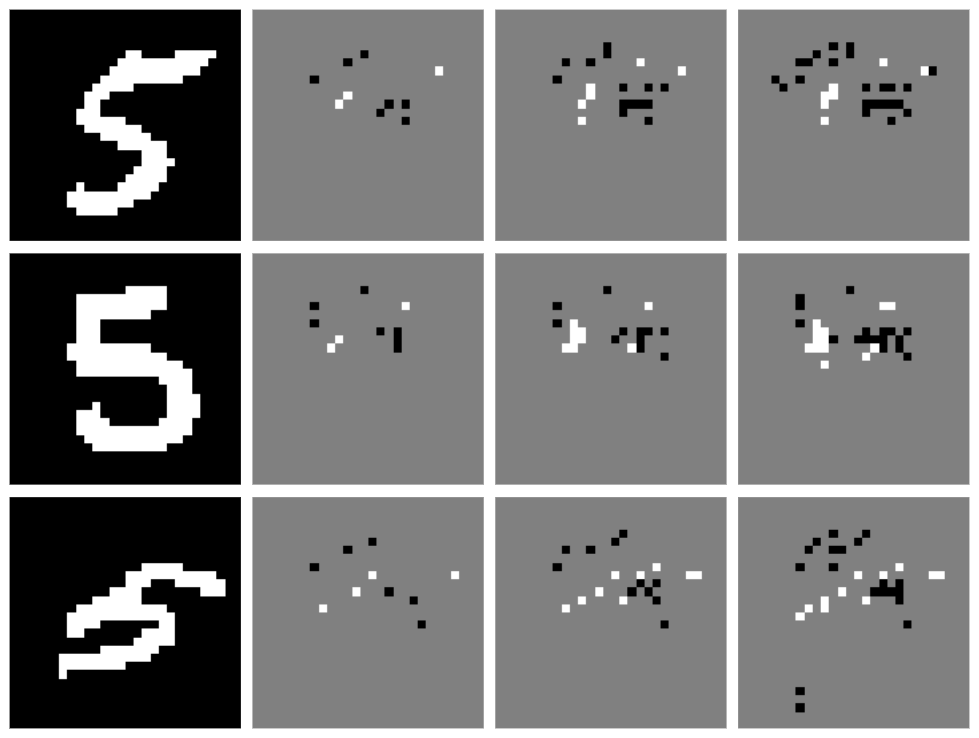}
        \subcaption{Correctly classified 5s}
    \end{subfigure}
    \begin{subfigure}{\linewidth}
        \centering
        \includegraphics[width=0.7\linewidth]{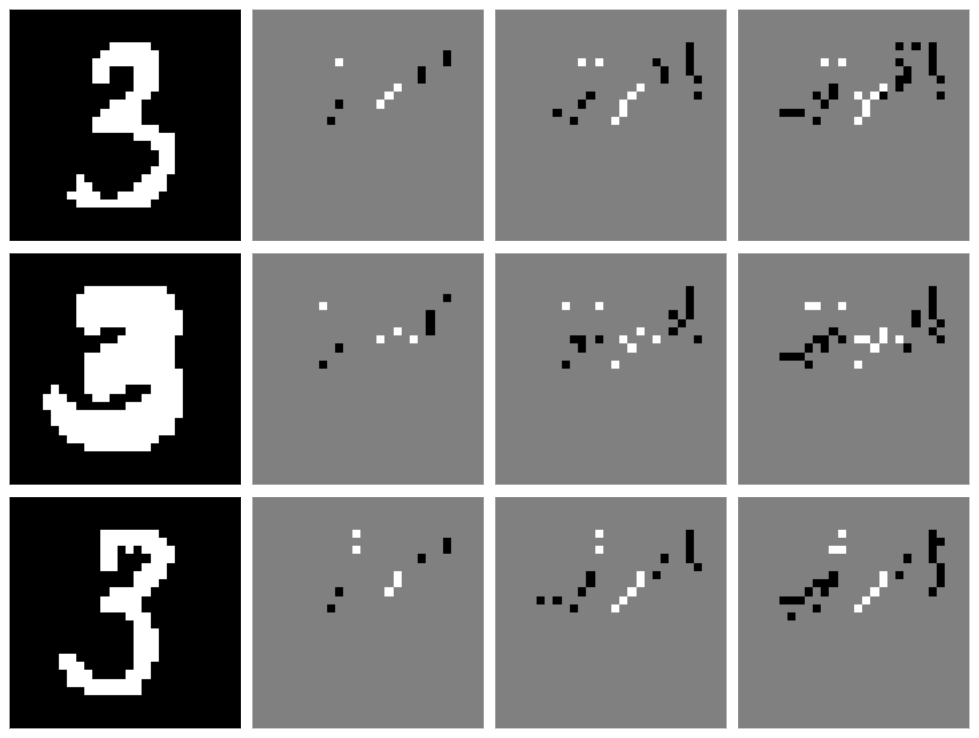}
        \subcaption{Correctly classified 3s}
    \end{subfigure}
    \begin{subfigure}{\linewidth}
        \centering
        \includegraphics[width=0.7\linewidth]{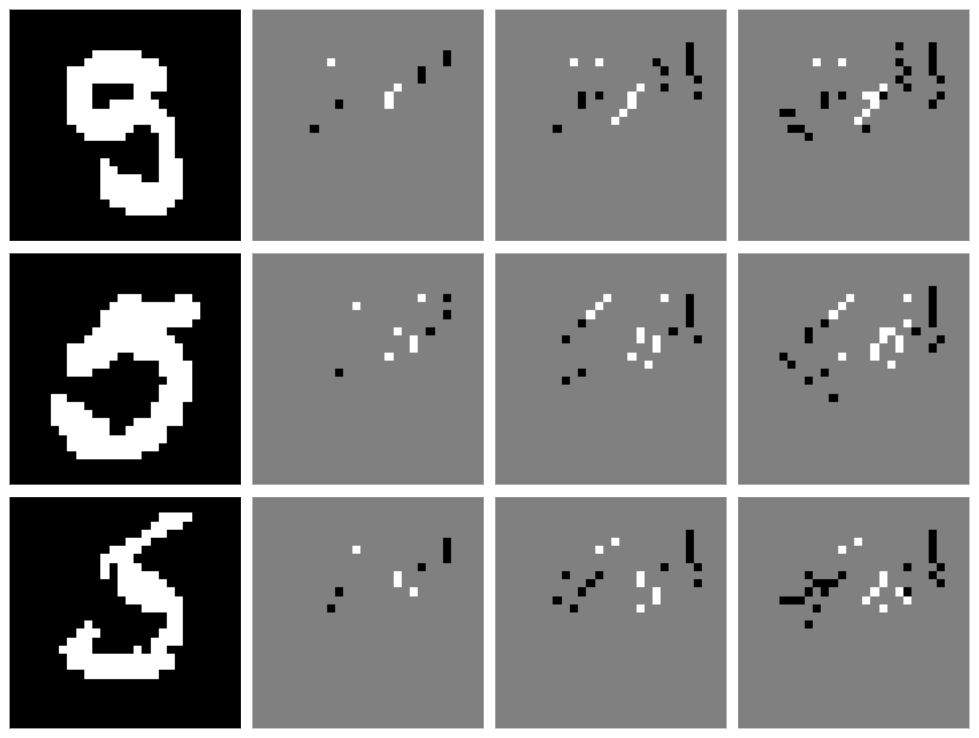}
        \subcaption{Misclassified 5s}
    \end{subfigure}
    \begin{subfigure}{\linewidth}
        \centering
        \includegraphics[width=0.7\linewidth]{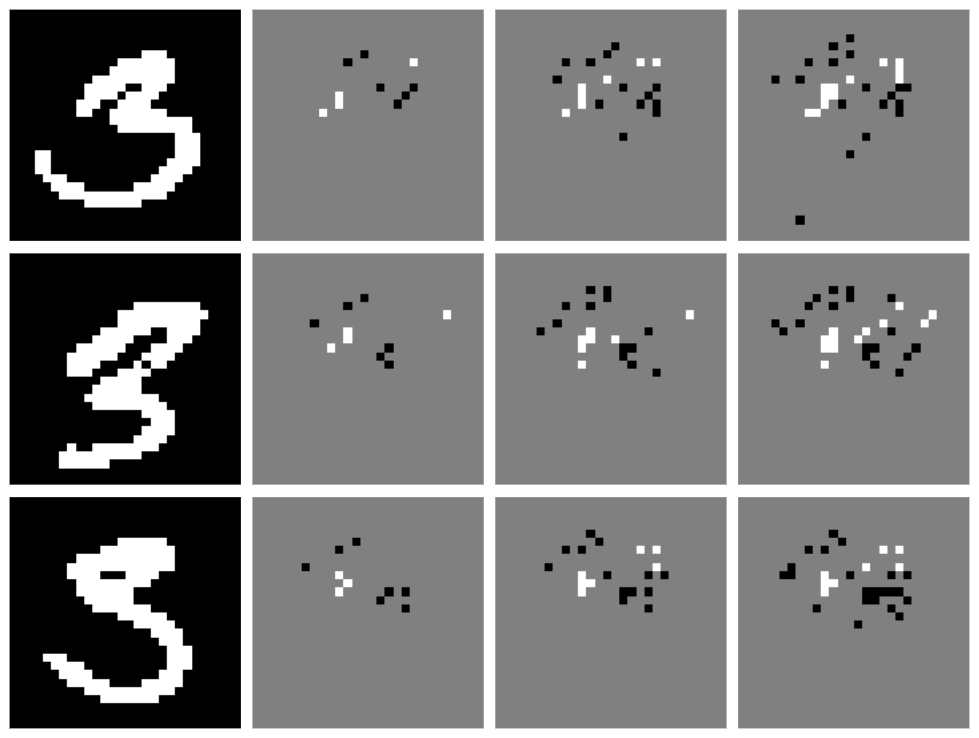}
        \subcaption{Misclassified 3s}
    \end{subfigure}
    \caption{More MNIST examples along with explanations with cardinality constraint $k=10,20,30$.}
    \label{fig:more_mnist}
\end{figure}

\end{document}